\documentclass{article}
\usepackage{amsmath}
\usepackage{amsfonts}
\usepackage{amssymb}
\usepackage{amsthm}
\usepackage{hyperref}
\usepackage{graphicx}
\usepackage{subcaption}
\usepackage{xcolor}
\usepackage{float}
\usepackage{tikz}
\usetikzlibrary{decorations.markings}

\usepackage{pgfplots}
\pgfmathdeclarefunction{gauss}{2}{%
  \pgfmathparse{1/(#2*sqrt(2*pi))*exp(-((x-#1)^2)/(2*#2^2))}%
}

\newtheorem{theorem}{{\bf Theorem}}
\newtheorem{example}{{\bf Example}}
\newtheorem{proposition}{{\bf Proposition}}
\newtheorem{definition}{{\bf Definition}}

\newtheorem{lemma}{{\bf Lemma}}
\newtheorem{remark}{{\bf Remark}}

\newcommand{\auc}{\rm  AUC}
\newcommand{\iauc}{\rm  IAUC}
\newcommand{\roc}{\rm  ROC}
\newcommand{\iroc}{\rm  IROC}
\def\vus{{\rm VUS}}
\def\X{{\mathcal X}}
\title{Ranking Data with Continuous Labels \\through Oriented Recursive Partitions}
\author{
  Stephan Cl\'emen\c{c}on \hspace{2em} Mastane Achab\\
  LTCI, T\'el\'ecom ParisTech\\
  75013 Paris, France\\
  \texttt{first.last@telecom-paristech.fr}\\
}
\begin{document}

\date{}
\maketitle

\begin{abstract} We formulate a supervised learning problem, referred to as \textit{continuous ranking}, where a continuous real-valued label $Y$ is assigned to an observable r.v.\ $X$ taking its values in a feature space $\mathcal{X}$ and the goal is to order all possible observations $x$ in $\mathcal{X}$ by means of a scoring function $s:\mathcal{X}\rightarrow \mathbb{R}$ so that $s(X)$ and $Y$ tend to increase or decrease together with highest probability. This problem generalizes \textit{bi/multi-partite ranking} to a certain extent and the task of finding optimal scoring functions $s(x)$ can be naturally cast as optimization of a dedicated functional criterion, called the $\iroc$ curve here, or as
maximization of the Kendall $\tau$ related to the pair $(s(X),Y)$. From the theoretical side, we describe the optimal elements of this problem and provide statistical guarantees for empirical Kendall $\tau$ maximization under appropriate conditions for the class of scoring function candidates. We also propose a recursive statistical learning algorithm tailored to empirical $\iroc$ curve optimization and producing a piecewise constant scoring function that is fully described by an oriented binary tree. Preliminary numerical experiments highlight the difference in nature between \textit{regression} and \textit{continuous ranking} and provide strong empirical evidence of the performance of empirical optimizers of the criteria proposed.

\end{abstract}

\section{Introduction}
The predictive learning problem considered in this paper can be easily stated in an informal fashion, as follows. Given a collection of objects of arbitrary cardinality, $N\geq 1$ say, respectively described by characteristics $x_1,\; \ldots,\; x_N$ in a feature space $\mathcal{X}$, the goal is to learn how to order them by increasing order of magnitude of a certain unknown continuous variable $y$. To fix ideas, the attribute $y$ can represent the 'size' of the object and be difficult to measure, as for the physical measurement of microscopic bodies in chemistry and biology or the cash flow of companies in quantitative finance and the features $x$ may then correspond to \textit{indirect measurements}. The most convenient way to define a preorder on a feature space $\mathcal{X}$ is to transport the natural order on the real line onto it by means of a (measurable) scoring function $s:\mathcal{X}\rightarrow \mathbb{R}$: an object with charcateristics $x$ is then said to be 'larger' ('strictly larger', respectively) than an object described by $x'$ according to the scoring rule $s$ when $s(x')\leq s(x)$ (when $s(x)<s(x')$). Statistical learning boils down here to build a scoring function $s(x)$, based on a \textit{training} data set $\mathcal{D}_n=\{(X_1,Y_1),\; \ldots,\; (X_n,Y_n)  \}$ of objects for which the values of all variables (direct and indirect measurements) have been jointly observed, such that $s(X)$ and $Y$ tend to increase or decrease together with highest probability or, in other words, such that the ordering of new objects induced by $s(x)$ matches that defined by their true measures as well as possible.
This problem, that shall be referred to as \textit{continuous ranking} throughout the article can be viewed as an extension of \textit{bipartite ranking}, where the output variable $Y$ is assumed to be binary and the objective can be naturally formulated as a functional $M$-estimation problem by means of the concept of $\roc$ curve, see \cite{cv09ieee}. Refer also to \cite{CLV05}, \cite{FISS03}, \cite{AGHHPR05} for approaches based on the optimization of summary performance measures such as the $\auc$ criterion in the binary context. Generalization to the situation where the random label is ordinal and may take a finite number $K\geq 3$ of values is referred to as \textit{multipartite ranking} and has been recently investigated in \cite{CRV13} (see also \textit{e.g.} \cite{Agarwal2}), where distributional conditions guaranteeing that $\roc$ surface and the $\vus$ criterion can be used to determine optimal scoring functions are exhibited in particular.

It is the major purpose of this paper to formulate the \textit{continuous ranking} problem in a quantitative manner and explore the connection between the latter and bi/multi-partite ranking. Intuitively, optimal scoring rules would be also optimal for any bipartite subproblem defined by thresholding the continuous variable $Y$ with cut-off $t>0$,  separating the observations $X$ such that $Y<t$ from those such that $Y>t$. Viewing this way \textit{continuous ranking} as a continuum of nested bipartite ranking problems, we provide here sufficient conditions for the existence of such (optimal) scoring rules and we introduce a concept of \textit{integrated $\roc$ curve} ($\iroc$ curve in abbreviated form) that may serve as a natural performance measure for continuous ranking, as well as the related notion of \textit{integrated $\auc$ criterion}, a summary scalar criterion, akin to Kendall tau. Generalization properties of empirical Kendall tau maximizers are discussed in the Supplementary Material.
The paper also  introduces a novel recursive algorithm that solves a discretized version of the empirical \textit{integrated $\roc$ curve} optimization problem,
producing a scoring function that can be computed by means of a hierarchical combination of binary classification rules. Numerical experiments providing strong empirical evidence of the relevance of the approach promoted in this paper are also presented.

The paper is structured as follows. The probabilistic framework we consider is described and key concepts of bi/multi-partite ranking are briefly recalled in section \ref{sec:preliminary}. Conditions under which optimal solutions of the problem of ranking data with continuous labels exist are next investigated in section \ref{sec:optimal}, while section \ref{sec:perf_measures} introduces a dedicated quantitative (functional) performance measure, the $\iroc$ curve. The algorithmic approach we propose in order to learn scoring functions with nearly optimal $\iroc$ curves is presented at length in section \ref{sec:rec_opt}. Numerical results are displayed in section \ref{sec:experiments}. Some technical proofs are deferred to the Supplementary Material.

\section{Notation and Preliminaries}\label{sec:preliminary}

Throughout the paper, the indicator function of any event $\mathcal{E}$ is denoted by $\mathbb{I}\{ \mathcal{E}\}$. The pseudo-inverse of any cdf $F(t)$ on $\mathbb{R}$ is denoted by $F^{-1}(u)=\inf\{ s\in\mathbb{R}:\; F(s)\geq u  \}$, while $\mathcal{U}([0,1])$ denotes the uniform distribution on the unit interval $[0,1]$.

\subsection{The probabilistic framework}
 Given a continuous real valued r.v.\ $Y$ representing an attribute of an object, its 'size' say, and a random vector $X$ taking its values in a (typically high dimensional euclidian) feature space $\mathcal{X}$ modelling other observable characteristics of the object (\textit{e.g.} 'indirect measurements' of the size of the object), hopefully useful for predicting $Y$, the statistical learning problem considered here is to
learn from $n\geq 1$ training independent observations $\mathcal{D}_n=\{ (X_1,Y_1),\; \ldots,\; (X_n,Y_n)\}$, drawn as  the pair $(X,Y)$, a measurable mapping $s:\mathcal{X}\rightarrow \mathbb{R}$, that shall be referred to as a \textit{scoring function} throughout the paper, so that the variables $s(X)$ and $Y$ tend to increase or decrease together: ideally, the larger the score $s(X)$, the higher the size $Y$. For simplicity, we assume throughout the article that $\mathcal{X}=\mathbb{R}^d$ with $d\geq 1$ and that the support of $Y$'s distribution is compact, equal to $[0,1]$ say. For any $q\geq 1$, we denote by $\lambda_q$ the Lebesgue measure on $\mathbb{R}^q$ equipped with its Borelian $\sigma$-algebra and suppose that the joint distribution $F_{X,Y}(dxdy)$ of the pair $(X,Y)$ has a density $f_{X,Y}(x,y)$ w.r.t.\ the tensor product measure $\lambda_d\otimes \lambda_1$. We also introduces the marginal distributions $F_Y(dy)=f_Y(y)\lambda_1(dy)$ and $F_X(dx)=f_X(x)\lambda_d(dx)$, where $f_Y(y)=\int_{x\in \mathcal{X}}f_{X,Y}(x,y)\lambda_d(dx)$ and  $f_X(x)=\int_{y\in [0,1]}f_{X,Y}(x,y)\lambda_1(dy)$ as well as the conditional densities $f_{X\mid Y=y}(x)=f_{X,Y}(x,y)  /f_Y(y)$ and $f_{Y\mid X=x}(y)=f_{X,Y}(x,y)  /f_X(x)$.
Observe incidentally that the probabilistic framework of the continuous ranking problem is quite similar to that of \textit{distribution-free regression}. However, as shall be seen in the subsequent analysis, even if the regression function $m(x)=\mathbb{E}[Y\mid X=x]$ can be optimal under appropriate conditions, just like for regression, measuring ranking performance involves criteria that are of different nature than the expected least square error and plug-in rules may not be relevant for the goal pursued here, as depicted by Fig. \ref{fig:least_squares} in the Supplementary Material.

\noindent {\bf Scoring functions.} The set of all scoring functions is denoted by $\mathcal{S}$ here. Any scoring function $s\in \mathcal{S}$ defines a total preorder on the space $\mathcal{X}$: $\forall (x,x')\in \mathcal{X}^2$, $x\preceq _s x' \Leftrightarrow s(x)\leq s(x')$. We also set $x\prec_s x'$ when $s(x)<s(x')$ and $x=_sx'$ when $s(x)=s(x')$ for $(x,x')\in \mathcal{X}^2$.

\subsection{Bi/multi-partite ranking} \label{subsec:bipartite}
Suppose that $Z$ is a binary label, taking its values in $\{-1,+1 \}$ say, assigned to the r.v.\ $X$. In bipartite ranking, the goal is to pick $s$ in $\mathcal{S}$ so that the larger $s(X)$, the greater the probability that $Y$ is equal to $1$ ideally. In other words, the objective is to learn $s(x)$ such that the r.v.\ $s(X)$ given $Y=+1$ is as \textit{stochastically larger}\footnote{Given two real-valued r.v.'s $U$ and $U'$, recall that $U$ is said to be \textit{stochastically larger} than $U'$ when $\mathbb{P}\{U\geq t  \}\geq \mathbb{P}\{U'\geq t  \}$ for all $t\in \mathbb{R}$.} as possible than the r.v.\ $s(X)$ given $Y=-1$: the difference between $\bar{G}_s(t)=\mathbb{P}\{s(X)\geq t \mid Y=+1  \}$ and $\bar{H}_s(t)=\mathbb{P}\{s(X)\geq t \mid Y=-1  \}$ should be thus maximal for all $t\in \mathbb{R}$. This can be naturally quantified by means of the notion of $\roc$ curve of a candidate $s\in \mathcal{S}$, \textit{i.e.}\ the parametrized curve $t\in \mathbb{R}\mapsto (\bar{H}_s(t),  \bar{G}_s(t))$, which can be viewed as the graph of a mapping $\roc_s:\alpha\in (0,1)\mapsto \roc_s(\alpha)$, connecting possible discontinuity points by linear segments (so that $\roc_s(\alpha)= \bar{G}_s\circ (1-H^{-1}_s)(1-\alpha)$ when $H_s$ has no flat part in $H^{-1}_s(1-\alpha)$, where $H_s=1-\bar{H}_s$). A basic Neyman Pearson's theory argument shows that the optimal elements $s^*(x)$ related to this natural (functional) bipartite ranking criterion (\textit{i.e.}\ scoring functions whose $\roc$ curve dominates any other $\roc$ curve everywhere on $(0,1)$) are transforms $(T\circ \eta)(x)$ of the posterior probability $\eta(x)=\mathbb{P}\{Z=+1 \mid X=x \}$, where $T:\textsc{supp} (\eta(X))\rightarrow \mathbb{R}$ is any strictly increasing borelian mapping. Optimization of the curve in $\sup$ norm has been considered in \cite{cv09ieee} or in \cite{CV09CA} for instance. However, given its functional nature, in practice the $\roc$ curve of any $s\in\mathcal{S}$ is often summarized by the area under it, which performance measure can be interpreted in a probabilistic manner, as the theoretical rate of \textit{concording pairs}
\begin{equation}
\auc(s)=\mathbb{P}\left\{  s(X)<s(X')\mid Z=-1,\; Z'=+1    \right\}+\frac{1}{2}\mathbb{P}\left\{  s(X)=s(X')\mid Z=-1,\; Z'=+1    \right\},
\end{equation}
where $(X',Z')$ denoted an independent copy of $(X,Z)$. A variety of algorithms aiming at maximizing the $\auc$ criterion or surrogate pairwise criteria have been proposed and studied in the literature, among which \cite{FISS03}, \cite{Rak04} or \cite{CDV13}, whereas generalization properties of empirical $\auc$ maximizers have been studied in \cite{CLV08}, \cite{AGHHPR05} and \cite{menon2016bipartite}.
An analysis of the relationship between the $\auc$ and the error rate is given in \cite{cortes2004auc}.

Extension to the situation where the label $Y$ takes at least three ordinal values (\textit{i.e.}\ multipartite ranking)  has been also investigated, see \textit{e.g.} \cite{Agarwal2} or \cite{CR14}. In \cite{CRV13}, it is shown that, in contrast to the bipartite setup, the existence of optimal solutions cannot be guaranteed in general and conditions on $(X,Y)$'s distribution ensuring that optimal solutions do exist and that extensions of bipartite ranking criteria such as the $\roc$ manifold and the volume under it can be used for learning optimal scoring rules have been exhibited. An analogous analysis in the context of continuous ranking is carried out in the next section.

\section{Optimal elements in ranking data with continuous labels}\label{sec:optimal}
In this section, a natural definition of the set of optimal elements for continuous ranking is first proposed. Existence and characterization of such optimal scoring functions are next discussed.

\subsection{Optimal scoring rules for continuous ranking}

Considering a threshold value $y\in[0,1]$, a considerably weakened (and discretized) version of the problem stated informally above would consist in finding $s$ so that the r.v.\ $s(X)$ given $Y>y$ is as stochastically larger than $s(X)$ given $Y<y$ as possible. This \textit{subproblem} coincides with the \textit{bipartite ranking} problem related to the pair $(X, Z_y)$, where $Z_y=2\mathbb{I}\{ Y>y \}-1$. As briefly recalled in subsection \ref{subsec:bipartite}, the optimal set $\mathcal{S}^*_y$ is composed of the scoring functions that induce the same ordering as
$$
\eta_y(X)=\mathbb{P}\{ Y>y\mid X \}=1-(1-p_y)/(1-p_y+p_y\Phi_y(X)),
$$
where $p_y=1-F_Y(y)=\mathbb{P}\{ Y>y \}$ and $\Phi_y(X)=(dF_{X\mid Y>y}/dF_{X\mid Y<y})(X)$.

\noindent{\bf A continuum of bipartite ranking problems.} The rationale behind the definition of the set $\mathcal{S}^*$ of optimal scoring rules for continuous ranking is that any element $s^*$ should score observations $x$ in the same order as $\eta_y$ (or equivalently as $\Phi_y$).

\begin{definition}\label{def:opt1}{\sc (Optimal scoring rule)} An optimal scoring rule for the continuous ranking problem related to the random pair $(X,Y)$ is any element $s^*$ that fulfills: $\forall y\in (0,1)$,
\begin{equation}\label{eq:opt}
\forall (x,x')\in \mathcal{X}^2,\;\; \eta_y(x)<\eta_y(x')\Rightarrow s^*(x)<s^*(x').
\end{equation}
In other words, the set of optimal rules is defined as $\mathcal{S}^*=\bigcap_{y\in (0,1)}\mathcal{S}^*_y$.
\end{definition}
It is noteworthy that, although the definition above is natural, the set $\mathcal{S}^*$ can be empty in absence of any distributional assumption, as shown by the following example.
\begin{example}
  As a counter-example, consider the distributions $F_{X,Y}$ such that
  $  F_Y=\mathcal{U}([0, 1])$ and $F_{X\mid Y=y} = \mathcal{N}(|2y-1|, (2y-1)^2)$.
Observe that $(X, 1-Y){\buildrel d \over =}(X, Y)$, so that $\Phi_{1-t} = \Phi_t^{-1}$ for all $t\in (0,1)$ and there exists $t\neq 0$ s.t. $\Phi_t$ is not constant. Hence, there exists no $s^*$ in $\mathcal{S}$ such that \eqref{eq:opt} holds true for all $t\in (0,1)$.

\end{example}
\begin{remark}\label{rk:invariance}{\sc (Invariance)}
We point out that the class $\mathcal{S}^*$ of optimal elements for continuous ranking thus defined is invariant by strictly increasing transform of the 'size' variable $Y$ (in particular, a change of unit has no impact on the definition of $\mathcal{S}^*$): for any borelian and strictly increasing mapping $H:(0,1)\rightarrow (0,1)$, any scoring function $s^*(x)$ that is optimal for the continuous ranking problem related to the pair $(X,Y)$ is still optimal for that related to $(X,H(Y))$ (since, under these hypotheses, for any $y\in (0,1)$: $Y>y\Leftrightarrow H(Y)>H(y)$).
\end{remark}
\subsection{Existence and characterization of optimal scoring rules}

We now investigate conditions guaranteeing the existence of optimal scoring functions for the continuous ranking problem.

\begin{proposition}\label{prop:charact}
The following assertions are equivalent.
\begin{enumerate}
\item For all $0<y<y'<1$, for all $(x,x')\in \mathcal{X}^2$:
$
\Phi_y(x)<\Phi_y(x')\Rightarrow \Phi_{y'}(x)\leq \Phi_{y'}(x').
$
\item There exists an optimal scoring rule $s^*$ (\textit{i.e.}\ $\mathcal{S}^*\neq \emptyset$).
\item The regression function $m(x)=\mathbb{E}[Y\mid X=x]$ is an optimal scoring rule.
\item The collection of probability distributions $F_{X\mid Y=y}(dx)=f_{X\mid Y=y}(x)\lambda_d(dx)$, $y\in (0,1)$ satisfies the monotone likelihood ratio property: there exist $s^*\in \mathcal{S}$ and, for all $0<y<y'<1$, an increasing function $\varphi_{y,y'}:\mathbb{R}\rightarrow \mathbb{R}_+$ such that: $\forall x\in \mathbb{R}^d$,
$$
\frac{f_{X\mid Y=y'}}{f_{X\mid Y=y}}(x)=\varphi_{y,y'}(s^*(x)).
$$
\end{enumerate}
\end{proposition}
Refer to the Appendix section for the technical proof. Truth should be said, assessing that Assertion $1.$ is a very challenging statistical task. However, through important examples, we now describe (not uncommon) situations where the conditions stated in Proposition \ref{prop:charact} are fulfilled.
\begin{example}\label{ex:prop1} We give a few important examples of probabilistic models fulfilling the properties listed in Proposition \ref{prop:charact}.
\smallskip

\noindent $\bullet$ {\bf Regression model.} Suppose that $Y=m(X)+\epsilon$, where $m:\mathcal{X}\rightarrow \mathbb{R}$ is a borelian function and $\epsilon$ is a centered r.v.\ independent from $X$. One may easily check that $m\in \mathcal{S}^*$.
\smallskip

\noindent $\bullet$ {\bf Exponential families.} Suppose that
$f_{X\mid Y=y}(x)=\exp(\kappa(y)T(x)-\psi(y))f(x)$ for all $x\in \mathbb{R}^d$,
where $f:\mathbb{R}^d\rightarrow \mathbb{R}_+$ is borelian, $\kappa:[0,1]\rightarrow \mathbb{R}$ is a borelian strictly increasing function and $T:\mathbb{R}^d\rightarrow \mathbb{R}$ is a borelian mapping such that $\psi(y)=\log \int_{x\in \mathbb{R}^d}\exp(\kappa(y)T(x))f(x)dx<+\infty$.

\end{example}

We point out that, although the regression function $m(x)$ is an optimal scoring function when $\mathcal{S}^*\neq \emptyset$, the \textit{continuous ranking} problem does not coincide with \textit{distribution-free regression} (notice incidentally that, in this case, any strictly increasing transform of $m(x)$ belongs to $\mathcal{S}^*$ as well). As depicted by Fig. \ref{fig:least_squares} the least-squares criterion is not relevant to evaluate continuous ranking performance and naive plug-in strategies should be avoided, see Remark \ref{rk:reg} below. Dedicated performance criteria are proposed in the next section.

\section{Performance measures for continuous ranking}\label{sec:perf_measures}

We now investigate quantitative criteria for assessing the performance in the continuous ranking problem, which practical machine-learning algorithms may rely on. We place ourselves in the situation where the set $\mathcal{S}^*$ is not empty, see Proposition \ref{prop:charact} above.

\noindent {\bf A functional performance measure.} It follows from the view developped in the previous section that, for any $(s,s^*)\in \mathcal{S}\times \mathcal{S}^*$ and for all $y\in (0,1)$, we have:
\begin{equation} \label{eq:cont_roc}
\forall\alpha\in (0,1),\;\;
\roc_{s,y}(\alpha)\leq \roc_{s^*,y}(\alpha)=\roc^*_y(\alpha),
\end{equation}
denoting by $\roc_{s,y}$ the $\roc$ curve of any $s\in\mathcal{S}$ related to the bipartite ranking subproblem $(X,Z_y)$ and by $\roc^*_y$ the corresponding optimal $\roc$ curve, \textit{i.e.}\ the $\roc$ curve of strictly increasing transforms of $\eta_y(x)$. Based on this observation, it is natural to design a dedicated performance measure by aggregating these 'sub-criteria'. Integrating over $y$ w.r.t.\ a $\sigma$-finite measure $\mu$ with support equal to $[0,1]$, this leads to the following definition $\iroc_{\mu,s}(\alpha)=\int \roc_{s,y}(\alpha)\mu(dy)$. The functional criterion thus defined inherits properties from the $\roc_{s,y}$'s (\textit{e.g.} monotonicity, concavity). In addition, the curve $\iroc_{\mu, s^*}$ with $s^*\in \mathcal{S}^*$ dominates everywhere on $(0,1)$ any other curve $\iroc_{\mu,s}$ for $s\in \mathcal{S}$. However, except in pathologic situations (\textit{e.g.} when $s(x)$ is constant), the curve $\iroc_{\mu,s}$ is not invariant when replacing $Y$'s distribution by that of a strictly increasing transform $H(Y)$. In order to guarantee that this desirable property is fulfilled (see Remark \ref{rk:invariance}), one should integrate w.r.t.\ $Y$'s distribution (which boils down to replacing $Y$ by the uniformly distributed r.v.\ $F_Y(Y)$).

\begin{definition}{\sc (Integrated $\roc/\auc$ criteria)} The \textit{integrated $\roc$} curve of any scoring rule $s\in \mathcal{S}$ is defined as: $\forall \alpha\in (0,1)$,
\begin{equation}\label{eq:iroc}
\iroc_{s}(\alpha)=\int_{y=0}^1\roc_{s,y}(\alpha)F_Y(dy)=\mathbb{E}\left[\roc_{s,Y}(\alpha) \right].
\end{equation}
The \textit{integrated $\auc$} criterion is defined as the area under the \textit{integrated $\roc$} curve: $\forall s\in \mathcal{S}$,
\begin{equation}\label{eq:iauc}
\iauc(s)=\int_{\alpha=0}^1\iroc_{s}(\alpha) d\alpha.
\end{equation}
\end{definition}
The following result reveals the relevance of the functional/summary criteria defined above for the continuous ranking problem. Additional properties of $\iroc$ curves are listed in the Supplementary Material.

 \begin{theorem}\label{thm:opt}
 Let $s^*\in \mathcal{S}$. The following assertions are equivalent.
 \begin{enumerate}
 \item The assertions of Proposition \ref{prop:charact} are fulfilled and $s^*$ is an optimal scoring function in the sense given by Definition \ref{def:opt1}.
 \item For all $\alpha\in(0,1)$, $\iroc_{s^*}(\alpha)= \mathbb{E}\left[ \roc^*_{Y}(\alpha) \right ]$.
 \item We have $\iauc_{s^*}= \mathbb{E}\left[ \auc^*_{Y} \right ]$, where $\auc^*_y=\int_{\alpha=0}^1\roc^*_y(\alpha)d\alpha$ for all $y\in (0,1)$.
 \end{enumerate}
If $\mathcal{S}^*\neq \emptyset$, then we have: $\forall s\in \mathcal{S}$,
\begin{eqnarray*}
\iroc_s(\alpha)&\leq & \iroc^*(\alpha)\overset{def}{=}\mathbb{E}\left[ \roc^*_{Y}(\alpha) \right ],\;\; \text { for  any }\alpha\in (0,1),\\
\iauc(s)&\leq& \iauc^*\overset{def}{=}\mathbb{E}\left[ \auc^*_{Y} \right ].
\end{eqnarray*}
In addition, for any borelian and strictly increasing mapping $H:(0,1)\rightarrow (0,1)$, replacing $Y$ by $H(Y)$ leaves the curves $\iroc_s$, $s\in \mathcal{S}$, unchanged.
 \end{theorem}

Equipped with the notion defined above, a scoring rule $s_1$ is said to be more accurate than another one $s_2$ if $\iroc_{s_2}(\alpha)\leq \iroc_{s_1}(\alpha)$ for all $\alpha\in (0,1)$.The $\iroc$ curve criterion thus provides a partial preorder on $\mathcal{S}$.
Observe also that, by virtue of Fubini's theorem, we have $
\iauc(s)=\int \auc_y(s) F_Y(dy)$ for all $s\in \mathcal{S}$,
denoting by $\auc_y(s)$ the $\auc$ of $s$ related to the bipartite ranking subproblem $(X,Z_y)$. Just like the $\auc$ for bipartite ranking, the scalar $\iauc$ criterion defines a full preorder on $\mathcal{S}$ for continuous ranking. Based on a training dataset $\mathcal{D}_n$ of independent copies of $(X,Y)$, statistical versions of the $\iroc/\iauc$ criteria can be straightforwardly computed by replacing the distributions $F_Y$, $F_{X\mid Y>t}$ and $F_{X\mid Y<t}$ by their empirical counterparts in \eqref{eq:cont_roc}-\eqref{eq:iauc}, see the Supplementary Material for further details. The lemma below  provides a probabilistic interpretation of the $\iauc$ criterion.
\begin{lemma}\label{lem:iauc}
Let $(X',Y')$ be a copy of the random pair $(X,Y)$ and $Y''$ a copy of the r.v.\ $Y$. Suppose that $(X,Y)$, $(X',Y')$ and $Y''$ are defined on the same probability space and are independent. For all $s\in \mathcal{S}$, we have:
\begin{equation}\label{eq:iauc_bis}
\iauc(s)=\mathbb{P}\left\{ s(X)<s(X')  \mid Y<Y''<Y'   \right\}+\frac{1}{2}\mathbb{P}\left\{ s(X)=s(X')  \mid Y<Y''<Y'   \right\}.
\end{equation}
\end{lemma}
This result shows in particular that a natural statistical estimate of $\iauc(s)$ based on $\mathcal{D}_n$ involves $U$-statistics of degree $3$. Its proof is given in the Supplementary Material for completeness.

\noindent {\bf The Kendall $\tau$ statistic.} The quantity \eqref{eq:iauc_bis} is akin to another popular way to measure the tendency to define the same ordering on the statistical population in a summary fashion:
\begin{eqnarray}\label{eq:Kendall}
d_{\tau}\left( s \right)&\overset{def}{=}&\mathbb{P}\left\{ \left(s(X)-s(X')  \right)\cdot \left( Y-Y'  \right) >0 \right\}+\frac{1}{2}\mathbb{P}\left\{ s(X)=s(X')  \right\}\label{eq:Kendall_true}
\\
&=& \mathbb{P}\{ s(X)<s(X') \mid Y<Y' \}+\frac{1}{2}\mathbb{P}\left\{ X=_sX'  \right\},\nonumber
\end{eqnarray}
where $(X',Y')$ denotes an independent copy of $(X,Y)$, observing that $\mathbb{P}\{ Y<Y' \}=1/2$. The empirical counterpart of \eqref{eq:Kendall_true} based on the sample $\mathcal{D}_n$, given by
\begin{equation}\label{eq:Kendall_emp}
\widehat{d}_n(s)=\frac{2}{n(n-1)}\sum_{i<j}  \mathbb{I}\left\{ \left(s(X_i)-s(X_j)  \right)\cdot \left(Y_i-Y_j  \right) >0\right\} +\frac{1}{n(n-1)}\sum_{i<j} \mathbb{I}\left\{s(X_i)=s(X_j)  \right\}
\end{equation}
 is known as the \textit{Kendall $\tau$ statistic} and is widely used in the context of statistical hypothesis testing. The quantity \eqref{eq:Kendall_true} shall be thus referred to as the (theoretical or true) \textit{Kendall $\tau$}.
 Notice that $d_{\tau}(s)$ is invariant by strictly increasing transformation of $s(x)$ and thus describes properties of the order it defines. The following result reveals that the class $\mathcal{S}^*$, when non empty, is the set of maximizers of the theoretical Kendall $\tau$. Refer to the Supplementary Material for the technical proof.

 \begin{proposition}\label{prop:Kendall_crit}
 Suppose that $\mathcal{S}^*\neq \emptyset$. For any $(s,s^*)\in \mathcal{S}\times \mathcal{S}^*$, we have: $d_{\tau}(s)\leq d_{\tau}(s^*)$.
 \end{proposition}

  Equipped with these criteria, the objective expressed above in an informal manner can be now formulated in a quantitative manner as a (possibly functional) $M$-estimation problem. In practice, the goal pursued is to find a reasonable approximation of a solution to the optimization problem
$ \max_{s\in \mathcal{S}}d_{\tau}(s)$ (respectively $\max_{s\in \mathcal{S}}\iauc(s)$),
 where the supremum is taken over the set of all scoring functions $s:\mathcal{X}\rightarrow \mathbb{R}$. Of course, these criteria are unknown in general, just like $(X,Y)$'s probability distribution, and the empirical risk minimization (ERM in abbreviated form) paradigm (see \cite{DGL96}) invites for maximizing the statistical version \eqref{eq:Kendall_emp} over a class $\mathcal{S}_0\subset \mathcal{S}$ of controlled complexity when considering the criterion $d_{\tau}(s)$ for instance.
The generalization capacity of empirical maximizers of the Kendall $\tau$ can be straightforwardly established using results in \cite{CLV08}. More details are given in the Supplementary Material.

Before describing a practical algorithm for recursive maximization of the $\iroc$ curve, a few remarks are in order.
\begin{remark}{\sc (On Kendall $\tau$ and $\auc$)} We point out that, in the bipartite ranking problem (\textit{i.e.}\ when the output variable $Z$ takes its values in $\{ -1,\; +1\}$, see subsection \ref{subsec:bipartite}) as well, the $\auc$ criterion can be expressed as a function of the Kendall $\tau$ related to the pair $(s(X), Z)$ when the r.v.\ $s(X)$ is continuous. Indeed, we have in this case $2p(1-p)\auc(s)=d_{\tau}(s)$, where $p=\mathbb{P}\{Z=+1 \}$ and $d_{\tau}(s)=\mathbb{P}\{  (s(X)-s(X'))\cdot (Z-Z')>0\}$, denoting by $(X',Z')$ an independent copy of $(X,Z)$.
\end{remark}
\begin{remark}\label{rk:reg}{\sc (Connection to distribution-free regression)} Consider the nonparametric regression model
$Y=m(X)+\epsilon$, where $\epsilon$ is a centered r.v.\ independent from $X$. In this case, it is well-known that the regression function $m(X)=\mathbb{E}[Y\mid X]$ is the (unique) solution of the expected least squares minimization. However, although $m\in \mathcal{S}^*$, the least squares criterion is far from appropriate to evaluate ranking performance, as depicted by Fig. \ref{fig:least_squares}. Observe additionally that, in contrast to the criteria introduced above, increasing transformation of the output variable $Y$ may have a strong impact on the least squares minimizer: except for linear stransforms, $\mathbb{E}[H(Y)\mid X]$ is not an increasing transform of $m(X)$. \end{remark}
\begin{remark}{\sc (On discretization)} Bi/multi-partite algorithms are not directly applicable to the continuous ranking problem. Indeed a discretization of the interval [0, 1] would be first required but this would raise a difficult question outside our scope: how to choose this discretization based on the training data? We believe that this approach is less efficient than ours which reveals problem-specific criteria, namely $\iroc$ and $\iauc$.
\end{remark}

\section{Continuous Ranking through Oriented Recursive Partitioning}\label{sec:rec_opt}

It is the purpose of this section to introduce the algorithm {\sc CRank}, a specific tree-structured learning algorithm for continuous ranking.

\subsection{Ranking trees and Oriented Recursive Partitions}\label{subsec:ranking_tree}
 Decision trees undeniably figure among the most popular techniques, in supervised and unsupervised settings, refer to \cite{cart84} or \cite{Quinlan} for instance. This is essentially due to the visual model summary they provide, in the form of a binary tree graphic that permits to describe predictions by means of a hierachichal combination  of elementary rules of the type "$X^{(j)}\leq \kappa$" or "$X^{(j)}>\kappa$", comparing the value taken by a (quantitative) component of the input vector $X$ (the \textit{split variable}) to a certain threshold (the \textit{split value}). In contrast to local learning problems such as classification or regression, predictive rules for a global problem such as \textit{ranking} cannot be described by a (tree-structured) partition of the feature space: cells (corresponding to the terminal leaves of the binary decision tree) must be ordered so as to define a scoring function. This leads to the definition of \textit{ranking trees} as binary trees equipped with a "left-to-right" orientation, defining a tree-structured collection of anomaly scoring functions, as depicted by Fig. \ref{fig:anom_tree}.
Binary ranking trees have been in the context of bipartite ranking in \cite{cv09ieee} or in  \cite{CDV13} and in \cite{CRV13} in the context of multipartite ranking.
The root node of a ranking tree $\mathcal{T}_J$ of depth $J\geq 0$ represents the whole feature space $\mathcal{X}$: $\mathcal{C}_{0,0}=\mathcal{X}$, while each internal node $(j,k)$ with $j<J$ and $k\in \{0,\; \ldots,\; 2^j -1 \}$ corresponds to a subset $\mathcal{C}_{j,k}\subset\X$, whose left and right siblings respectively correspond to disjoint subsets $\mathcal{C}_{j+1, 2k}$ and $\mathcal{C}_{j+1, 2k+1}$ such that $\mathcal{C}_{j,k}=\mathcal{C}_{j+1, 2k}\cup \mathcal{C}_{j+1, 2k+1}$. Equipped with the left-to-right orientation, any subtree $\mathcal{T}\subset \mathcal{T}_J$ defines a preorder on $\mathcal{X}$: elements lying in the same terminal cell of $\mathcal{T}$ being equally ranked. The scoring function related to the oriented tree $\mathcal{T}$ can be written as:

\begin{equation}\label{eq:score_tree}
s_{\mathcal{T}}(x)=\sum_{\mathcal{C}_{j,k}:\text{ terminal leaf of }\mathcal{T}}2^J\left (1-\frac{k}{2^j}\right)\cdot \mathbb{I}\{x\in \mathcal{C}_{j,k}  \}.
\end{equation}

 \begin{figure}[t]
\centering
 \includegraphics[height=4cm]{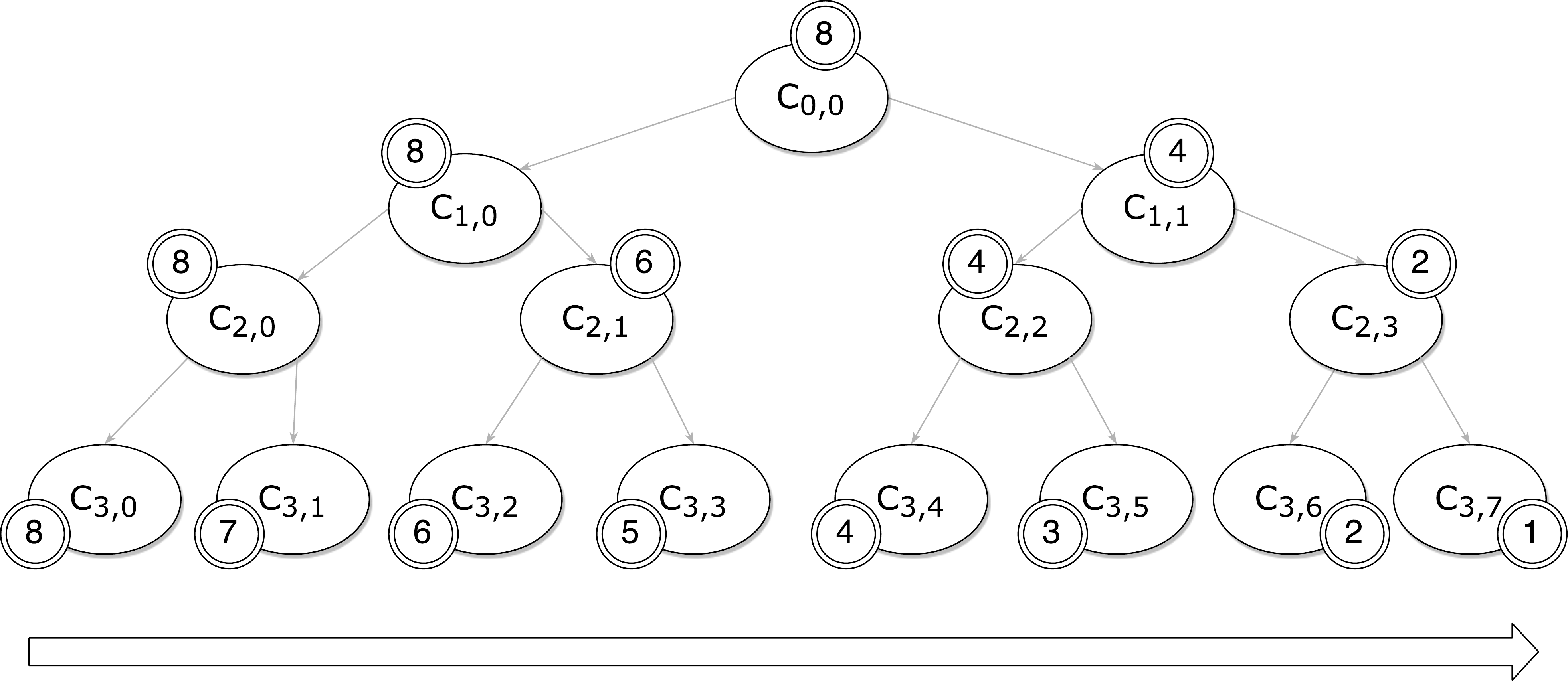}
 \caption{A scoring function described by an oriented binary subtree $\mathcal{T}$. For any element $x\in \mathcal{X}$, one may compute the quantity $s_{\mathcal{T}}(x)$ very fast in a top-down fashion by means of the heap structure: starting from the initial value $2^J$ at the root node, at each internal node $\mathcal{C}_{j,k}$, the score remains unchanged if $x$ moves down to the left sibling, whereas one subtracts $2^{J-(j+1)}$ from it if $x$ moves down to the right.}
 \label{fig:anom_tree}
 \end{figure}

\subsection{The {\sc CRank} algorithm}

Based on Proposition \ref{prop:Kendall_crit}, as mentioned in the Supplementary Material, one can try to build from the training dataset $\mathcal{D}_n$ a ranking tree by recursive empirical Kendall $\tau$ maximization. We propose below an alternative tree-structured recursive algorithm, relying on a (dyadic) discretization of the 'size' variable $Y$. At each iteration, the local sample (\textit{i.e.}\ the data lying in the cell described by the current node) is split into two halves (the highest/smallest halves, depending on $Y$) and the algorithm calls a binary classification algorithm $\mathcal{A}$ to learn how to divide the node into right/left children. The theoretical analysis of this algorithm and its connection with approximation of $\iroc^*$ are difficult questions that will be adressed in future work.
Indeed we found out that the $\iroc$ cannot be represented as a parametric curve contrary to the $\roc$, which renders proofs much more difficult than in the bipartite case.

\fbox{
\begin{minipage}[t]{13.5cm}
\medskip

{\small
\begin{center}
{\sc The CRank Algorithm}
\end{center}

\begin{enumerate}
\item {\bf Input.} Training data $\mathcal{D}_n$, depth $J\geq 1$, binary classification algorithm $\mathcal{A}$.
\item {\bf Initialization.} Set $\mathcal{C}_{0,0}=\mathcal{X}$.
\item {\bf Iterations.} For $j=0,\; \ldots,\; J-1$ and $k=0,\; \ldots,\; 2^J-1$,
\begin{enumerate}
\item Compute a median $y_{j,k}$ of the dataset $\{Y_1,\; \ldots,\;, Y_n  \}\cap \mathcal{C}_{j,k}$ and assign the binary label $Z_i=2\mathbb{I}\{ Y_i>y_{j,k}\}-1$ to any data point $i$ lying in $\mathcal{C}_{j,k}$, \textit{i.e.}\ such that $X_i\in \mathcal{C}_{j,k}$.
\item Solve the binary classification problem related to the input space $\mathcal{C}_{j,k}$ and the training set $\{(X_i,Y_i):\; 1\leq i \leq n,\; X_i\in \mathcal{C}_{j,k}  \}$, producing a classifier $g_{j,k}:\mathcal{C}_{j,k}\rightarrow \{-1,\; +1  \}$.
\item Set $\mathcal{C}_{j+1,2k}=\{x\in \mathcal{C}_{j,k},\;  g_{j,k}=+1  \}=\mathcal{C}_{j,k}\setminus\mathcal{C}_{j+1,2k+1}$.
\end{enumerate}
\item {\bf Output.} Ranking tree $\mathcal{T}_J=\{\mathcal{C}_{j,k}:\; 0\leq j\leq J,\; 0\leq k< D  \}$.
\end{enumerate}
}
\medskip

\end{minipage}
}

Of course, the depth $J$ should be chosen such that $2^J\leq n$. One may also consider continuing to split the nodes until the number of data points within a cell has reached a minimum specified in advance. In addition, it is well known that recursive partitioning methods fragment the data and the unstability of splits increases with the depth. For this reason, a ranking subtree must be selected. The growing procedure above should be classically followed by a pruning stage, where children of a same parent are progressively merged until the root $\mathcal{T}_0$ is reached and a subtree among the sequence $\mathcal{T}_0\subset \ldots\subset \mathcal{T}_J$ with nearly maximal $\iauc$ should be chosen using cross-validation. Issues related to the implementation of the {\sc CRank} algorithm and variants (\textit{e.g.} exploiting randomization/aggregation) will be investigated in a forthcoming paper.

\section{Numerical Experiments}\label{sec:experiments}

In order to illustrate the idea conveyed by Fig. \ref{fig:least_squares} that the least squares criterion is not appropriate for the continuous ranking problem
we compared on a toy example \textsc{CRank} with \textsc{CART}. Recall that the latter is a regression decision tree algorithm which minimizes the MSE (Mean Squared Error).
We also runned an alternative version of \textsc{CRank} which maximizes the empirical Kendall $\tau$ instead of the empirical $\iauc$: this method is refered to as \textsc{Kendall} from now on.
The experimental setting is composed of a unidimensional feature space $\mathcal{X}=[0, 1]$ (for visualization reasons) and a simple regression model without any noise: $Y=m(X)$.
Intuitively, a least squares strategy can miss slight oscillations of the regression function, which are critical in ranking when they occur in high probability regions as they affect the order among the feature space.
The results are presented in Table \ref{tab:performance}. See Supplementary Material for further details.

\begin{table}[H]
\begin{center}
\begin{tabular}{l | c c c}
\  & $\iauc$ & Kendall $\tau$ & MSE \\
\textsc{CRank} & $\color{red}{0.95}$ & $0.92$ & $0.10$ \\
\textsc{Kendall} & $0.94$ & $\color{red}{0.93}$ & $0.10$ \\
\textsc{CART} & $0.61$ & $0.58$ & $\color{red}{7.4 \times 10^{-4}}$ \\
\end{tabular}
\vspace{.1in}
\caption{$\iauc$, Kendall $\tau$ and MSE empirical measures}\label{tab:performance}
\end{center}
\end{table}

\section{Conclusion}\label{sec:concl}

This paper considers the problem of learning how to order objects by increasing 'size', modeled as a continuous r.v.\ $Y$, based on \textit{indirect measurements} $X$. We provided a rigorous mathematical formulation of this problem that finds many applications (\textit{e.g.} quality control, chemistry) and is referred to as \textit{continuous ranking}. In particular, necessary and sufficient conditions on $(X,Y)$'s distribution for the existence of optimal solutions are exhibited and appropriate criteria have been proposed for evaluating the performance of scoring rules in these situations.
In contrast to distribution-free regression where the goal is to recover the local values taken by the regression function, \textit{continuous ranking} aims at reproducing the preorder it defines on the feature space as accurately as possible.
The numerical results obtained via the algorithmic approaches we proposed for optimizing the criteria aforementioned highlight the difference in nature between these two statistical learning tasks.

\section*{Acknowledgments}
This work was supported by the industrial chair \textit{Machine Learning for Big Data} from T\'el\'ecom ParisTech
and by a public grant (\textit{Investissement d'avenir} project, reference ANR-11-LABX-0056-LMH, LabEx LMH).

\bibliographystyle{plain}

\bibliography{Ranking}

\newpage

\section*{Appendix - Technical Proofs}

\subsection*{Proof of Proposition \ref{prop:charact}}
Observe first that $3. \Rightarrow 2.$ and $1.\Leftrightarrow 4.$ are obvious.\\
$2. \Rightarrow 1.$: Let us assume that assertion $2.$ is true. Let $(x,x')\in \mathcal{X}^2$ and $y\in (0, 1)$ such that $\Phi_y(x)<\Phi_y(x')$. Then, from assumption $2.$, $s^*(x)<s^*(x')$.
For $t'\in (y, 1)$, if $\Phi_{y'}(x)>\Phi_{y'}(x')$, it leads to the following contradiction: $s^*(x)>s^*(x')$. Hence $\Phi_{y'}(x)\le \Phi_{y'}(x')$.\\
$1. \Rightarrow 3.$: Let us assume that assertion $1.$ is true. Let $(x,x')\in \mathcal{X}^2$ and $y\in (0, 1)$ such that $\eta_y(x)<\eta_y(x')$.
Observe that $(x, y')\mapsto \eta_{y'}(x)$ is continuous.
It follows from assumption $1.$ that for $y'\in (0, 1)$, $\eta_{y'}(x)\le \eta_{y'}(x')$ with strict inequality on a nonempty interval by continuity of $(x, y')\mapsto \eta_{y'}(x)$.
Integrating the latter inequality against the uniform distribution over $(0, 1)$ leads to $m(x)<m(x')$.

\subsection*{Proof of Theorem \ref{thm:opt}}
The implications $1. \Rightarrow 2.$ and $2. \Rightarrow 3.$ are obvious.\\
$3. \Rightarrow 1.$: Let us assume that assertion $3.$ is true. Assume ad absurdum that $1.$ is false. Then there exists $y\in (0,1)$ s.t. $\auc_y(s^*) < \auc_y(\eta_y)$.
Notice that $(x, y')\mapsto \eta_{y'}(x)$ and, for any scoring function $s$, $y'\mapsto \auc_{y'}(s)$ are continuous.
By integration w.r.t.\ $F_Y$ we obtain $\iauc(s^*)<\mathbb{E}\left[ \auc^*_{Y} \right ]$, which contradicts assertion $3$.
Hence $1.$ is true.

\subsection*{Proof of Lemma \ref{lem:iauc}}
Recall that, for any $s\in \mathcal{S}$ and all $y\in (0,1)$, we have:
$$
\auc_y(s)=\mathbb{P}\left\{  s(X)<s(X')\mid Y<y<Y'    \right\}+\frac{1}{2}\mathbb{P}\left\{  s(X)=s(X')\mid Y<y<Y'    \right\}.
$$
Integrating the terms in the equation above w.r.t.\ $F_Y(dy)$ leads to the desired formula.
Then, a natural empirical version of $\iauc(s)$ is:
\begin{equation*}
  \begin{split}
    \widehat{\iauc}_n(s) &= \frac{6}{n(n-1)(n-2)}\sum_{(i, j, k)}  \mathbb{I}\left\{ s(X_i)<s(X_k), Y_i<Y_j<Y_k \right\}\\
    &\quad+\frac{3}{n(n-1)(n-2)}\sum_{(i, j, k)} \mathbb{I}\left\{s(X_i)=s(X_k), Y_i<Y_j<Y_k  \right\}.
  \end{split}
\end{equation*}
The asymptotic and nonasymptotic study of the deviation of $\widehat{\iauc}_n$ will be the subject of future work.

\subsection*{Proof of Proposition \ref{prop:Kendall_crit}}
We assume that $s(X)$ is a continuous r.v.\ for simplicity, the slight modifications needed to extend the argument to the general framework being left to the reader. As a first go, observe that
\begin{eqnarray*}
d_{\tau}(s)&=& \mathbb{P}\{s(X')>s(X)\mid Y'>Y  \}=\int_{y'=0}^1\mathbb{P}\left\{s(X')>s(X)  \mid Y'=y',\; Y<y'\right\}  F_Y(dy')
\end{eqnarray*}
Notice next that, for any $y'\in (0,1)$, $\mathbb{P}\left\{s(X')>s(X)  \mid Y'=y',\; Y<y'\right\}$ is nothing else than the $\auc$ criterion of $s(x)$ related to the distribution of $X$ given $Y<y'$ (negative distribution) and $F_{X\mid Y=y'}$ (positive distribution). Since we assumed $\mathcal{S}^*\neq \emptyset$, the collection $\{F_{X\mid Y=y}:\; y\in (0,1)  \}$ is of increasing likelihood ratio and according to Theorem \ref{prop:charact}, any $s^*\in \mathcal{S}^*$
is a Neyman Pearson test statistic and thus defines uniformly most powerful tests (among unbiased tests) of $\mathcal{H}_0:Y<y$ against $\mathcal{H}_1:Y=y$. Hence, for any $y'\in (0,1)$, $\mathbb{P}\left\{s(X')>s(X)  \mid Y'=y',\; Y<y'\right\}\leq \mathbb{P}\left\{s^*(X')>s^*(X)  \mid Y'=y',\; Y<y'\right\}$. Integrating over $y'$ w.r.t.\ $F_Y$ yields the desired result.

\subsection*{On Empirical Kendall $\tau$ Maximization}
Here we state a result describing the performance of scoring rules obtained through maximization of the empirical Kendall $\tau$ over a class $\mathcal{S}_0\subset \mathcal{S}$ of controlled complexity. An empirical $Kendall$ $\tau$ maximizer over $\mathcal{S}_0$ is any scoring function $\widehat{s}_n\in \mathcal{S}_0$ s.t.
\begin{equation}
\widehat{d}_n(\widehat{s}_n)=\max_{s\in \mathcal{S}_0}\widehat{d}_n(s).
\end{equation}
\begin{theorem} Suppose that $\mathcal{S}^*\neq \emptyset$ and set $d_{\tau}^*=d_{\tau}(s^*)$ for $s^*\in \mathcal{S}^*$. Assume that $\mathcal{S}_0$ is a {\rm VC} major class of functions with {\rm VC} dimension $V<+\infty$. Let $\delta\in (0,1)$. With probability at least $1-\delta$, we have:
\begin{equation}
d^*_{\tau}-d_{\tau}(\widehat{s}_n)\leq c\sqrt{\frac{V}{n}}+4\sqrt{\frac{\log(1/\delta)}{n-1}} +\left\{ d_{\tau}^*- \max_{s\in \mathcal{S}_0}d_{\tau}(s)  \right\}.
\end{equation}
\end{theorem}
\begin{proof}
The argument is based on the simple bound
$$
d^*_{\tau}-d_{\tau}(\widehat{s}_n)\leq 2\sup_{s\in \mathcal{S}_0} \left\vert \widehat{d}_n(s)-d_{\tau}(s) \right\vert    +\left\{ d_{\tau}^*- \max_{s\in \mathcal{S}_0}d_{\tau}(s)  \right\},
$$
combined with the use of concentration results for the $U$-process $\{ \widehat{d}_n(s)-d_{\tau}(s)  \}_{s\in }\mathcal{S}_0$. The proof is finished by mimicking that of Corollary 3 in \cite{CLV08}.
\end{proof}
From a computational perspective, maximizing $\widehat{d}_n$ is a challenge, the optimization problem being NP-hard due to the absence of convexity/smoothness of the pairwise loss function $\mathbb{I}\{ (s(x)-s(x'))(y-y')>0 \}$. Whereas replacing this loss by a surrogate loss, more suited to continuous optimization, is a possible strategy, using greedy algorithms in the spirit of the popular {\sc CART} method can also be considered for this purpose.  A slight modification of {\sc CART} based on recursive maximization of the empirical Kendall $\tau$ criterion (rather than the Gini index or the least squares criterion) permit to build an \textit{oriented ranking tree} in a top down manner, see subsection \ref{subsec:ranking_tree}. Just like for classification/regression, the procedure can be followed by a pruning stage (model selection), based here on (\textit{e.g. }cross-validation based) estimates of Kendall $\tau$.

\section*{Appendix - Additional Remarks}

\subsection*{Properties of $\iroc$ curves}

For any scoring function $s\in\mathcal{S}$ and $y\in(0, 1)$, we define the conditional cdfs of $s(X)$ as follows:
\begin{equation*}
  H_{s, y}(v) = \mathbb{P}(s(X)\le v \mid Y<y),
\end{equation*}
\begin{equation*}
  G_{s, y}(v) = \mathbb{P}(s(X)\le v \mid Y>y).
\end{equation*}
Now we give some properties of the $\iroc$ curve which are easily derived from $\roc$ curve properties by integration over bipartite ranking subproblems.
\begin{theorem} For any scoring function $s\in\mathcal{S}$, the following properties hold:
\begin{itemize}
\item \textbf{Limit values.} We have $\iroc_s(0)=0$ and $\iroc_s(1)=1$.
\item \textbf{Invariance.} For any strictly increasing funciton $T: \mathbb{R}\rightarrow\mathbb{R}$, we have for all $\alpha\in (0, 1)$, $\iroc_{T\circ s}(\alpha)=\iroc_{s}(\alpha)$.
\item \textbf{Concavity.} If for all $y\in (0, 1)$ the likelihood ratio $dG_{s, y}/dH_{s, y}$ is a monotone function, then the $\iroc$ curve is concave.
\end{itemize}
\end{theorem}

\begin{proof}
  Use Proposition 24 in \cite{cv09ieee} for each bipartite ranking subproblem at level $y\in(0, 1)$. Then integrate over $y$ w.r.t.\ $F_Y$.
\end{proof}

\subsection*{Distribution-free regression vs continuous ranking}
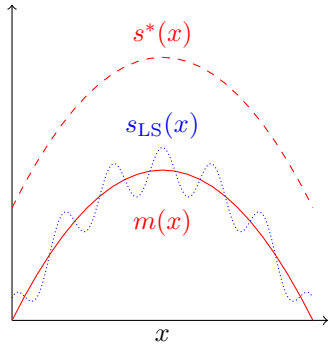
\begin{figure}[H]
\centering
  \begin{tikzpicture}[domain=0:2,label/.style={%
   postaction={ decorate,
   decoration={ markings, mark=at position .5 with \node #1;}}}]
  \draw [help lines] (-2,0) (2,4);
  \draw [->] (-2,0) -- (0, 0) node (xaxis) [below] {$x$} node[above=1cm,pos=1,red] {$m(x)$} -- (2.2,0);
  \draw [->] (-2,-0) -- (-2, 4.2);
  \draw [domain=-2:2, red] plot (\x, {0.5*(4-\x*\x)});
  \draw [dashed, domain=-2:2, red, label={[above]{$s^*(x)$}}] plot (\x, {0.5*(4-\x*\x)+1.5});
  \draw [densely dotted, domain=-2:2, samples=500, blue, label={[above]{$s_{\text{LS}}(x)$}}] plot (\x, {0.5*(4-\x*\x)+0.3*cos(3*pi*\x r)});
  \end{tikzpicture}
  \caption{The least squares regressor $s_{\text{LS}}$ (dotted line) better approximates, in terms of mean squared error, the regression function $m$ (solid line) than $s^*$ (dashed line) does. Still, the latter is optimal for the ranking task as it is a strictly increasing transform of $m$.}
  \label{fig:least_squares}
\end{figure}

\subsection*{Numerical Experiments (Figures)}

We considered a polynomial regression function $m$ over $[0, 1]$ and valued in $[0, 1]$, namely:
$$
m(x) = \frac{P(x)-P(0)}{P(1)-P(0)},
$$
where the polynomial function $P$ is given by:
$$
P(x) = z^2\cdot(z+1)\cdot(z+1.5)\cdot(z+2), \quad \text{with} \quad z=25\cdot(x-0.5).
$$
Observe that $m$ slightly oscillates in the interval $I_2=[0.415, 0.51]$ (see \ref{fig:polynomial_zoom}).
With respective probabilities $p_1=0.1$, $p_2=0.8$ and $p_3=0.1$, $X$ is uniformly sampled in one of the three intervals $I_1=[0, 0.415]$, $I_2$ and $I_3=[0.51, 1]$: the critical window $I_2$ is then a high probability region.
The three algorithms (\textsc{CRank}, \textsc{Kendall} and \textsc{CART}) where trained on the same dataset $(X_1, Y_1), \dots, (X_{n_\text{train}}, Y_{n_\text{train}})$ with $Y_i=m(X_i)$ and $n_\text{train}=100$ with the same constraint on the depth of the tree: at most $D = 3$.
Then we tested them on $n_\text{test}=2000$ new iid copies of $X$.
In Fig. \ref{fig:simu} we plot the polynomial function $m$ and piecewise constant scoring functions provided by the three approaches.

We observe in Fig. \ref{fig:simu} that \textsc{CRank} and \textsc{Kendall} almost provide the same ranking functions ($s_{\textsc{CRank}} \approx s_{\textsc{Kendall}}$) and achieve similar performance (see Fig. \ref{tab:performance}).
Also notice in Fig. \ref{tab:performance} that \textsc{CRank}, \textsc{Kendall} and \textsc{CART} respectively achieve maximum $\iauc$, Kendall $\tau$ and MSE.
As expected, \textsc{CART} misses the critical oscillations that is why its $\iauc$ and Kendall $\tau$ are considerably lower than for its concurrents.

\begin{figure}[H]
\begin{subfigure}{.5\textwidth}
  \centering
  \includegraphics[height=.8\linewidth]{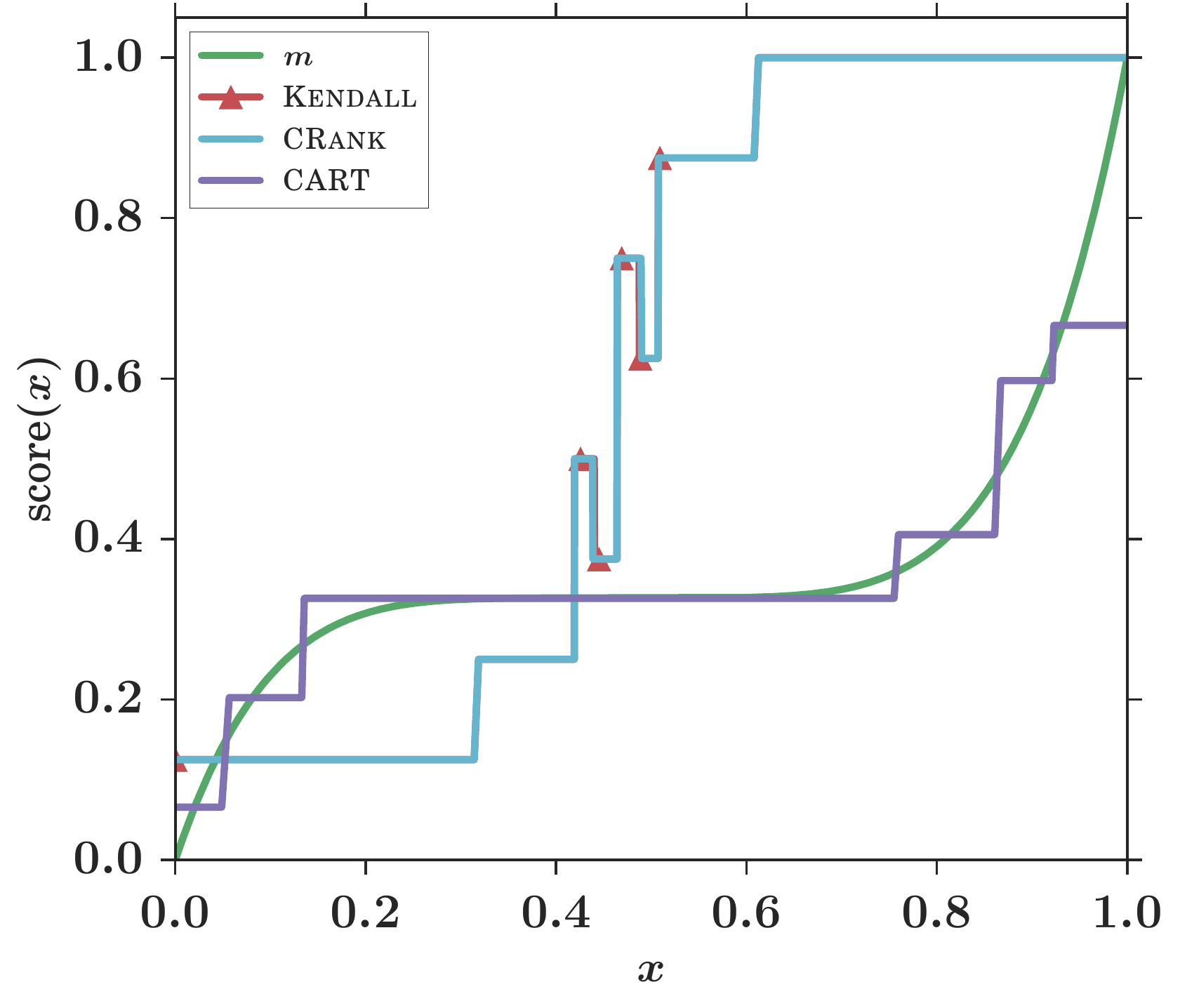}
  \caption{}
  \label{fig:polynomial}
\end{subfigure}
\begin{subfigure}{.5\textwidth}
  \centering
  \includegraphics[height=.85\linewidth]{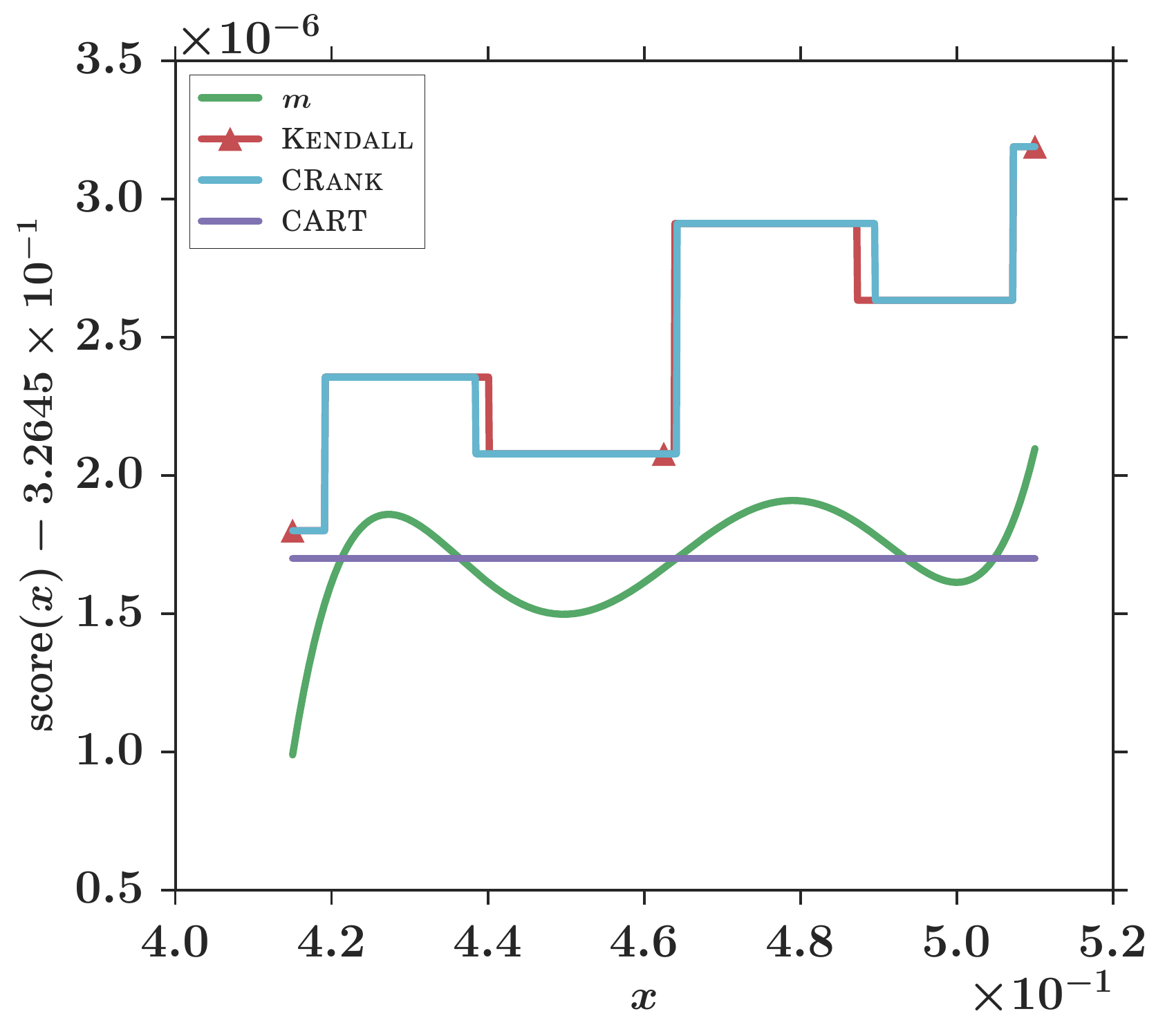}
  \caption{}
  \label{fig:polynomial_zoom}
\end{subfigure}
\caption{Polynomial regression function $m$ and scoring functions provided by \textsc{CRank}, \textsc{Kendall} and \textsc{CART}.
For visualization reasons, $s_{\textsc{CRank}}$ and $s_{\textsc{Kendall}}$ have been renormalized by $2^D=8$ to take values in $[0, 1]$ and, in Fig. \ref{fig:polynomial_zoom}, affine functions have been applied to the three scoring functions.}
\label{fig:simu}
\end{figure}

\end{document}